\documentclass[conference]{IEEEtran}
\IEEEoverridecommandlockouts
\usepackage{cite}
\usepackage{amsmath,amssymb,amsfonts}
\usepackage{amsthm}
\newtheorem{proposition}{Proposition}
\usepackage{algorithmic}
\usepackage{graphicx}
\usepackage{textcomp}
\usepackage{xcolor}
\usepackage{booktabs}
\usepackage{url}
\usepackage{tikz}
\usetikzlibrary{positioning,arrows.meta,calc}

\newcommand{\modelname}{Monotone FedNAM}

\begin{document}

\title{Physically Constrained Federated Additive Models for O-RAN SLA-Risk Prediction}

\author{
    \IEEEauthorblockN{
        Aubida A. Al-Hameed\IEEEauthorrefmark{2},
        Mohammed M. H. Qazzaz\IEEEauthorrefmark{1}\IEEEauthorrefmark{2},
        Maryam Hafeez\IEEEauthorrefmark{1}, and
        Syed A. Zaidi\IEEEauthorrefmark{1}
    }
    \IEEEauthorblockA{
        \IEEEauthorrefmark{1}School of Electronic and Electrical Engineering, University of Leeds, Leeds, UK \\
        Email: \{ml14mmh, m.hafeez, s.a.zaidi\}@leeds.ac.uk
    }
    \IEEEauthorblockA{
        \IEEEauthorrefmark{2}College of Electronics Engineering, Ninevah University, Mosul, Iraq \\
        Email: aubida.alhameed@uoninevah.edu.iq
    }
}
\maketitle

\begin{abstract}
Proactive service assurance in O-RAN requires predicting per-slice SLA
violations before they occur. The prediction model must be auditable by
operators and must train across base stations without pooling per-slice
KPIs, which are commercially sensitive because slices are leased to
individual tenants. Neural additive models (NAMs) offer auditability
because each KPI contributes through a visible shape function. However,
visibility alone does not guarantee physical validity. On the ColO-RAN
testbed dataset, unconstrained NAMs learn effects that contradict wireless
physics, for example predicting higher risk when channel quality improves.
This failure appears under both local and centralized training, and
non-IID federated averaging worsens it. We present \modelname{}, a
federated additive model in which KPIs with unambiguous physical direction
are represented as monotone splines whose constraints survive FedAvg
aggregation by construction, while contestable KPIs remain unconstrained.
The model trains and operates as a Non-RT RIC rApp and is compact enough
for deployment as a Near-RT RIC xApp. \modelname{} eliminates all
monotonicity violations, raises constrained shape consistency from 0.71 to
1.00, generalizes to an unseen scheduling policy, and reduces uplink
traffic by 65\%, at a cost of 0.04 to 0.07 AUC. These results show that
physically constrained federated additive models can support auditable
SLA risk inference for multi-tenant O-RAN service assurance.
\end{abstract}



\section{Introduction and Related Work}
\label{sec:intro}
\label{sec:related}
The Open Radio Access Network (O-RAN) turns the RAN into a software-defined, data-driven system controlled through RAN Intelligent Controllers (RICs)~\cite{polese2023understanding}. The non-real-time RIC (Non-RT RIC) hosts rApps for management and model training, and the near-real-time RIC (Near-RT RIC) hosts xApps for online inference and control~\cite{polese2023understanding}. This architecture makes proactive service assurance feasible. Rather than reacting to a breached service-level agreement (SLA), an operator can predict per-slice SLA violations a short horizon ahead and act before they occur~\cite{foukas2023ran}. SLA-violation prediction in O-RAN is intrinsically distributed, and the trust boundaries run inside a single network. Network slices are leased to enterprises and mobile virtual network operators~\cite{samdanis2016slicebroker,gsma2025ng116}, so per-slice key performance indicators (KPIs) expose a tenant's traffic patterns, load, and delivered service quality. These records are commercially sensitive toward other tenants, and the analytics rApp that consumes them may itself come from a third-party vendor. Shared and neutral-host deployments add a second boundary, since several operators can serve traffic through the same base stations. Pooling raw per-slice traces at one point is therefore often contractually or legally blocked, and centralizing high-rate KPI streams adds transport overhead as a secondary cost. Federated learning (FL) fits this setting. Each base station acts as a federated client and trains on its local per-slice records, while only model updates leave the site~\cite{mcmahan2017fedavg,kairouz2021advances}. FL has been applied in O-RAN to slicing, traffic steering, and anomaly detection~\cite{zhang2022federated,erdol2022federated,attanayaka2023p2p}, yet the resulting models remain hard to inspect. Accuracy alone is not sufficient for an ML-based xApp executing in the RIC. Operators must understand and audit the underlying model before its predictions drive a control loop, yet the prevailing solutions, such as deep reinforcement learning for slice resource allocation~\cite{raftopoulos2024latency} and recurrent demand forecasting~\cite{foukas2023ran}, are opaque~\cite{brik2024xai,khan2024explainable}. Most explainability work for the RAN attaches post-hoc explanations to black-box models~\cite{fiandrino2023explora,rezazadeh2023explanation}; such explanations can be unfaithful and give no guarantee that learned effects obey wireless physics and are computationally prohibitive at Near-RT RIC inference cadences. An attractive alternative is an intrinsically interpretable predictor. Neural Additive Models (NAMs)~\cite{agarwal2021nam,hastie1986gam} learn one shape function per feature, so the effect of each KPI on the prediction is visible by construction. We find that this visibility is not enough. On measured O-RAN traces, an unconstrained additive model learns shape functions that are often physically wrong. It can learn that better channel quality raises SLA-violation risk, which inverts basic wireless behavior. The cause is the data, not the model class. RAN KPIs are strongly correlated, since backlog, granted resources, delivered rate, and channel quality move together, so a flexible additive fit can attribute a confounded effect to the wrong feature's shape. The problem appears even under centralized training, and non-independent and identically distributed (non-IID) federated averaging compounds it by merging heterogeneous per-client shapes into a distorted global effect. Such violations undermine the audit either way. A blatant inversion is caught in review, but it destroys operator trust in the model. A partial violation over part of a shape function is easy to miss and can silently mislead. The natural remedy is to enforce the known physical directions as hard constraints~\cite{sill1997monotonic,liu2020certified,runje2023monotonic} rather than hope they emerge from the data. In a federated system this is not automatic, because a constraint satisfied by every client must also survive aggregation at the server. Providing exactly this guarantee is the core of our approach.
 
We address this with \modelname{}, a federated neural additive model for trustworthy SLA-violation risk prediction. Unlike prior federated NAMs~\cite{nanda2025fednam}, which leave feature effects unconstrained, \modelname{} preserves monotonicity under aggregation by construction. Each base station trains per-feature shape functions on its private KPIs, and the server aggregates them without breaking the constraints. Constraints are applied only to KPIs whose physical direction is unambiguous, and contestable KPIs are left free. SLA labels follow per-slice service profiles rather than ad hoc thresholds. We evaluate \modelname{} on the ColO-RAN testbed dataset~\cite{coloran2022}, where seven base stations form a realistic non-IID federation.
The contributions are as follows.
(1) We show that interpretability does not imply physical validity. On measured O-RAN traces, unconstrained neural additive models contradict expected KPI-risk directions in every tested regime, and non-IID federated averaging increases the violation rate. This shows that monotonicity should be imposed as an explicit wireless-domain prior.
(2) We propose \modelname{}, a federated neural additive model that preserves monotonicity under aggregation by construction. It constrains only KPIs with unambiguous conditional direction, leaves contestable KPIs free by default, and retains per-feature interpretability.
(3) We define an O-RAN rApp-to-xApp deployment path for trustworthy SLA-risk inference. \modelname{} is trained and governed as a Non-RT RIC rApp, and the resulting global model can be deployed to a Near-RT RIC xApp for online per-slice SLA-risk prediction with auditable additive explanations.

\section{Problem Formulation and Method }
\label{sec:system}
\label{sec:method}

We predict the probability that a slice violates its SLA in the next window. The pipeline is direct: raw 4\,Hz per-slice KPMs are windowed over 2.5\,s, pooled across the slice's active users into a feature vector $\mathbf{x}$, and passed through the additive model and a sigmoid to yield the SLA-risk score. The current window supplies the features, and the following window supplies the label.

\subsection{Federated O-RAN Setting and Target}
We treat each base station as a federated client. Client $k$ owns a local dataset $\mathcal{D}_k$
\begin{equation}
\mathcal{D}_k = \{(\mathbf{x}_{k,n}, y_{k,n})\}_{n=1}^{N_k},
\label{eq:dataset}
\end{equation}
where $\mathbf{x}_{k,n}\in\mathbb{R}^{d}$ is a vector of per-slice RAN KPIs and $y_{k,n}\in\{0,1\}$ is the per-slice SLA-violation label for the next prediction window. The SLA-violation risk is $\Pr(y=1\mid\mathbf{x})$, the probability that a slice breaches its SLA in the next window given current KPIs. Clients share a common feature schema but differ in distribution, because one base station may face heavier load or worse channels than another. This is a horizontal federation: the server coordinates training but never receives raw client records. The confidentiality unit is the tenant slice: per-slice records must not leave the serving site in raw form, and the coordinating rApp receives only anchored model updates. Some KPIs carry a known physical direction. When the other KPIs are held fixed, higher backlog, demand, or error rate should not lower predicted risk, and better channel quality should not raise it. We encode this as a monotonic prior on those features.

\subsection{Additive Prediction Model}
A NAM decomposes the predictor into a sum of feature-specific functions~\cite{agarwal2021nam}:
\begin{equation}
f_\theta(\mathbf{x}) = \beta_0 + \sum_{j=1}^{d} g_j(x_j;\theta_j),
\label{eq:nam}
\end{equation}
where $g_j(\cdot)$ is a small subnetwork for feature $j$, $\theta_j$ are its parameters, and $\beta_0$ is an intercept. Each feature effect can be plotted as a one-dimensional shape function and checked against engineering expectations.

\subsection{Monotone Shape Functions}
Let $\mathcal{M}^{+}$ and $\mathcal{M}^{-}$ denote the features that should be non-decreasing and non-increasing with SLA-violation risk. \modelname{} enforces
\begin{equation}
\frac{\partial g_j(x_j)}{\partial x_j}\ge 0,\ j\in\mathcal{M}^{+},
\qquad
\frac{\partial g_j(x_j)}{\partial x_j}\le 0,\ j\in\mathcal{M}^{-}.
\end{equation}
Constrained features are represented as monotone linear splines on a shared grid of $Q=16$ knots. The knots are placed at marginal feature quantiles, agreed once before training from low-sensitivity per-feature order statistics rather than raw KPIs, and are kept identical across clients. Each spline stores raw parameters $\theta$, and its increments are $\delta=\mathrm{softplus}(\theta)$, where $\mathrm{softplus}(x)=\log(1+e^{x})>0$, so every increment is strictly positive. The grid size $Q{=}16$ balances shape resolution against update size. Quantile placement concentrates the knots where the data lie, so 16 knots resolve the KPI effects at the granularity relevant for risk prediction, while each constrained feature is transmitted as only 16 increment values per round, which is the source of the uplink savings reported in Section~\ref{sec:results}. The feature effect is the signed cumulative sum of these increments, so it is non-decreasing for $j\in\mathcal{M}^{+}$ and non-increasing for $j\in\mathcal{M}^{-}$ by construction, with few parameters and a direct guarantee on the grid. Unconstrained continuous features use small MLP subnetworks, and categorical features use linear terms. The hybrid design keeps guarantees where the direction is known and neural flexibility where it is not.

\subsection{Anchoring for Federated Identifiability}
Additive models are identifiable only up to constant offsets. Adding a constant $c_j$ to any $g_j$ and subtracting $\sum_j c_j$ from $\beta_0$ leaves every prediction in~\eqref{eq:nam} unchanged, so the data determine each shape function only up to its vertical level. In a single centralized model this ambiguity is harmless. Under federation it is not, for two reasons. First, the server averages client shapes pointwise, so arbitrary per-client levels leak into the aggregated shape and intercept and distort the global feature effect. Second, level-shifted client shapes are not comparable, so an auditor cannot tell whether clients disagree about a KPI effect or only about its offset. Each shape function is therefore anchored at a fixed public reference $a_j=0$:
\begin{equation}
\tilde{g}_{k,j}(x) = g_{k,j}(x) - g_{k,j}(a_j).
\label{eq:anchor}
\end{equation}
For standardized continuous features, $a_j=0$ is the common normalization reference, and for one-hot columns it is the absent level. The anchored constant $g_{k,j}(a_j)$ is removed from every client shape and folds into the global intercept $\beta_0$. The anchor is public, identical across clients, and needs no raw data exchange.

\subsection{Federated Objective and Training}
We use a sigmoid output with positive-class-weighted binary cross-entropy for the Bernoulli target~\cite{goodfellow2016deep}. The predictor $f_\theta(\mathbf{x})$ produces a logit, and
\begin{equation}
p_\theta(y=1\mid\mathbf{x}) = \sigma(f_\theta(\mathbf{x})),
\label{eq:sigmoid}
\end{equation}
where $\sigma(\cdot)$ is the logistic function and $\theta$ collects all trainable parameters. Following FedAvg~\cite{mcmahan2017fedavg}, the federated objective is
\begin{equation}
\min_{\theta}\ \sum_{k=1}^{K}\frac{N_k}{N}\,\mathcal{L}_k(\theta),
\qquad N=\sum_{k=1}^{K}N_k,
\label{eq:fedobj}
\end{equation}
with the local objective
\begin{equation}
\mathcal{L}_k(\theta)=\frac{1}{N_k}\sum_{n=1}^{N_k}
\ell_k\!\left(y_{k,n},p_\theta(y=1\mid\mathbf{x}_{k,n})\right),
\label{eq:localloss}
\end{equation}
and the weighted loss
\begin{equation}
\ell_k(y,p)=-\alpha_k y\log p-(1-y)\log(1-p).
\label{eq:weighted_bce}
\end{equation}
Here $\alpha_k=N_{k,0}^{\mathrm{tr}}/N_{k,1}^{\mathrm{tr}}$ is the positive-class weight from client $k$'s training split, where $N_{k,0}^{\mathrm{tr}}$ and $N_{k,1}^{\mathrm{tr}}$ are the negative and positive training counts. In implementation, the loss is applied to logits in the standard numerically stable form.

Training proceeds in communication rounds. At round $t$, the server broadcasts the current global NAM parameters to participating clients. Each client trains locally for $E{=}3$ epochs on its private data under the selected monotonicity constraints, anchors its feature functions using~\eqref{eq:anchor}, and sends its update to the server. The server then aggregates each transmitted feature block by sample-weighted averaging:
\begin{equation}
\psi_j^{(t+1)} = \sum_{k=1}^{K}\frac{N_k}{N}\,\psi_{k,j}^{(t+1)},
\label{eq:fedavg_step}
\end{equation}
where $\psi_j$ denotes the transmitted parameter block for feature $j$. For constrained features, $\psi_j$ is the signed spline-increment vector. For unconstrained subnetworks and the intercept, $\psi_j$ denotes the corresponding trainable parameter vector.

\begin{proposition}[Monotonicity preservation under spline-increment FedAvg] \label{prop:mono} Assume all clients use the same feature scaling, spline knots, and signed-increment parameterization for each constrained feature. If FedAvg~\eqref{eq:fedavg_step} is applied to the signed spline-increment vectors and each client satisfies its sign constraint, then the aggregated feature function preserves the same monotonic direction on the shared grid. \end{proposition} \begin{proof} A weighted average of non-negative increments is non-negative, and likewise for non-positive, so the aggregated spline preserves the required direction on the shared grid. \end{proof}

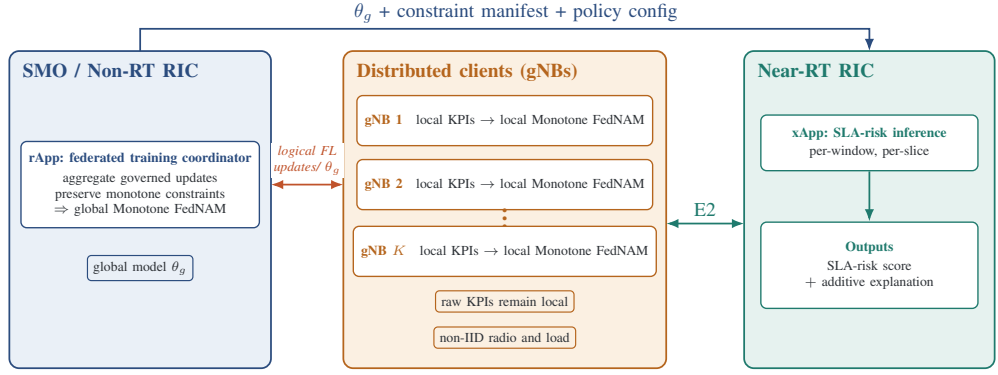
\begin{figure*}[t]
\centering
\definecolor{cA}{HTML}{2B4C7E}
\definecolor{cAbg}{HTML}{EAF0F7}
\definecolor{cB}{HTML}{B5651D}
\definecolor{cBbg}{HTML}{FBF0E2}
\definecolor{cC}{HTML}{1F7A6D}
\definecolor{cCbg}{HTML}{E6F2EF}
\definecolor{cAcc}{HTML}{C0532B}
\definecolor{cInk}{HTML}{2F3338}

\resizebox{0.72\textwidth}{!}{%
\begin{tikzpicture}[
  font=\sffamily,>=Latex,
  ptitle/.style={font=\bfseries\normalsize,anchor=north west},
  card/.style={rounded corners=3pt,draw,line width=0.7pt,align=center,inner sep=4pt,font=\scriptsize,text=cInk},
  chip/.style={rounded corners=2pt,draw,line width=0.6pt,align=center,inner sep=3pt,font=\scriptsize,text=cInk},
  arr/.style={-{Latex[length=2.2mm,width=1.7mm]},line width=0.9pt},
  darr/.style={{Latex[length=2.2mm,width=1.7mm]}-{Latex[length=2.2mm,width=1.7mm]},line width=0.9pt},
  lab/.style={font=\normalsize,fill=white,inner sep=1.5pt,align=center},
  ilab/.style={font=\scriptsize\itshape,fill=white,inner sep=1.5pt,align=center,text=cAcc}
]

\draw[rounded corners=5pt,line width=0.9pt,draw=cA,fill=cAbg] (0,0.3) rectangle (4.7,6.0);
\draw[rounded corners=5pt,line width=0.9pt,draw=cB,fill=cBbg] (6.0,0.3) rectangle (11.8,6.0);
\draw[rounded corners=5pt,line width=0.9pt,draw=cC,fill=cCbg] (13.2,0.3) rectangle (17.7,6.0);

\node[ptitle,text=cA] at (0.12,5.9) {SMO / Non-RT RIC};
\node[ptitle,text=cB] at (6.12,5.9) {Distributed clients (gNBs)};
\node[ptitle,text=cC] at (13.32,5.9) {Near-RT RIC};

\node[card,draw=cA,fill=white,minimum width=4.1cm,minimum height=1.55cm] (rapp) at (2.35,3.6)
{\textbf{\textcolor{cA}{rApp: federated training coordinator}}\\[2pt]
aggregate governed updates\\
preserve monotone constraints\\
$\Rightarrow$ global \modelname{}};

\node[chip,draw=cA,fill=cAbg] (theta) at (2.35,2.1)
{global model $\theta_g$};

\node[card,draw=cB,fill=white,minimum width=5.0cm,minimum height=0.9cm] (c1) at (8.9,4.75)
{\textbf{\textcolor{cB}{gNB 1}}\quad local KPIs $\to$ local \modelname{}};

\node[card,draw=cB,fill=white,minimum width=5.0cm,minimum height=0.9cm] (c2) at (8.9,3.6)
{\textbf{\textcolor{cB}{gNB 2}}\quad local KPIs $\to$ local \modelname{}};

\node[text=cB,font=\large\bfseries] at (8.9,3.1) {$\vdots$};

\node[card,draw=cB,fill=white,minimum width=5.0cm,minimum height=0.9cm] (cK) at (8.9,2.4)
{\textbf{\textcolor{cB}{gNB $K$}}\quad local KPIs $\to$ local \modelname{}};

\node[chip,draw=cB,fill=cBbg] at (8.9,0.85)
{non-IID radio and load};
\node[chip,draw=cB,fill=cBbg] at (8.9,1.5)
{raw KPIs remain local};
\node[card,draw=cC,fill=white,minimum width=3.9cm,minimum height=1.0cm] (xapp) at (15.45,4.35)
{\textbf{\textcolor{cC}{xApp: SLA-risk inference}}\\[1pt]
per-window, per-slice};

\node[card,draw=cC,fill=white,minimum width=3.9cm,minimum height=1.55cm] (out) at (15.45,2.15)
{\textbf{\textcolor{cC}{Outputs}}\\[2pt]
SLA-risk score\\
$+$ additive explanation};

\coordinate (smoE)   at (4.7,3.6);
\coordinate (cliW)   at (6.0,3.6);

\coordinate (smoTop) at (2.35,6.0);
\coordinate (nrtTop) at (15.45,6.0);

\coordinate (cliE)   at (11.8,2.9);
\coordinate (nrtW)   at (13.2,2.9);

\draw[darr,cAcc] (smoE) -- node[ilab,above=3pt]{logical FL\\updates/ $\theta_g$} (cliW);

\draw[arr,cA] (smoTop) -- (2.35,6.4) -- (15.45,6.4) -- (nrtTop);
\node[lab,text=cA] at (9.1,6.68) {$\theta_g$ + constraint manifest + policy config};

\draw[darr,cC] (nrtW) -- node[lab,text=cC,above=2pt]{E2} (cliE);

\draw[arr,cC] (xapp.south) -- (out.north);

\end{tikzpicture}}
\caption{The Non-RT RIC rApp governs federated training across gNB clients and outputs a governed model.}

\label{fig:oran_placement}
\end{figure*}

\subsection{O-RAN Placement}
Fig.~\ref{fig:oran_placement} shows the intended rApp-to-xApp deployment path. The Non-RT RIC hosts \modelname{} as an rApp for federated training and model governance. The rApp aggregates governed client updates and outputs a model package containing the global parameters $\theta_g$, the monotonicity-constraint manifest, and inference policy configuration. The Near-RT RIC hosts the deployed inference function as an xApp. The xApp does not receive raw training KPIs or federated updates. It uses the deployed model package with live KPMs to compute per-slice SLA-risk scores and additive explanations.

\section{Experimental Setup}
\label{sec:experiment}

\subsection{ColO-RAN Dataset}
We use the ColO-RAN dataset~\cite{coloran2022}, collected on the Colosseum O-RAN testbed in the Rome static-medium scenario. The dataset contains measurements from seven base stations and three slices: enhanced Mobile Broadband (eMBB, slice~0), Machine-Type Communications (MTC, slice~1), and Ultra-Reliable Low-Latency Communications (URLLC, slice~2). Each base station uses a 10\,MHz channel with 50 physical resource blocks. Three scheduling policies are available: round-robin (RR), waterfilling (WF), and proportionally fair (PF). For each scheduler, the dataset sweeps 28 resource-block-group allocations across the slices and repeats every configuration over five experiments. Per-slice key performance metrics (KPMs) are reported every 250\,ms (4\,Hz). This structure provides a natural federated setting: each base station acts as one client and observes a different local radio and traffic distribution, with the three slices standing in for tenant services whose records must not be pooled raw.

\subsection{Features, Constraints, and Labels}
Each client builds per-slice feature vectors over fixed, non-overlapping windows of $W{=}10$ samples (2.5\,s). Within a window, each user's KPI stream is first reduced to its temporal mean. The per-user values are then pooled across the slice's active users: volume and throughput KPIs are summed, and quality KPIs are averaged. The requested and granted PRB KPIs are cumulative counts over the 250\,ms reporting interval rather than instantaneous occupancy, so their window totals exceed the 50-PRB channel capacity by construction. Table~\ref{tab:mono_design} assigns the resulting ColO-RAN KPIs to conservative, aggressive, and unconstrained groups. The conservative set contains KPIs with unambiguous conditional direction and is constrained by default. The aggressive set contains KPIs whose direction is plausible but contestable; these features are left free by default and constrained only in the ablation of Section~\ref{sec:results}.

The aggressive features are left unconstrained by default because their effect on risk is less clear. Allocation KPIs, such as \texttt{slice\_prb} and \texttt{sum\_granted\_prbs}, are controlled by the dataset sweep and may saturate after demand is met. The current delivered rate is useful for prediction, but it correlates strongly with the next-window label and may introduce leakage. The feature \texttt{num\_ues} is also sensitive because the service profiles are defined on a per-user basis. The conservative set also includes user-level radio KPIs joined from the per-user logs. Radio-quality KPIs are constrained as non-increasing in risk, while block-error KPIs are constrained as non-decreasing in risk.

The prediction target is the next-window per-slice SLA-violation status. We use a one-window horizon ($H{=}1$), so each feature window predicts the immediately following 2.5\,s window. The labels are defined using service-profile proxies derived from the ColO-RAN measurements, standing in for per-tenant SLA contracts. The eMBB, MTC, and URLLC slices play the role of three tenant services with distinct assurance targets. An eMBB slice is positive when the per-user delivered rate is below 2\,Mbps, which is half of its 4\,Mbps offered load. An MTC slice is positive when the per-user delivered rate is below 0.015\,Mbps, which is half of its 0.03\,Mbps Poisson source rate. A URLLC slice is positive when the downlink buffer exceeds the backlog implied by a 1\,ms latency target, computed using Little's law, or when the downlink error rate exceeds 5\%. In this dataset, the buffer threshold is about 28.8\,bytes. A $\pm20\%$ sensitivity check on the eMBB and MTC thresholds produced stable violation rates. Table~\ref{tab:violation_rates} reports the resulting positive rates for each split. Features come from the current window and labels from the next. No future sample enters the feature vector, and every label window strictly post-dates its feature window. Monotone spline features use the shared 16-knot quantile grid described in Section~\ref{sec:method}, placed at marginal feature quantiles and applied identically across clients.
\begin{table}[t]
\centering
\caption{Monotonicity design for ColO-RAN per-slice SLA-violation prediction.}
\label{tab:mono_design}
\small
\renewcommand{\arraystretch}{1.05}
\resizebox{\columnwidth}{!}{
\begin{tabular}{@{}p{0.42\linewidth}p{0.41\linewidth}p{0.26\linewidth}@{}}
\toprule
KPI  & Effect on risk & Monotonic constraint \\
\midrule
\multicolumn{3}{@{}l}{\emph{Conservative set (constrained by default)}}\\
\texttt{dl\_buffer [bytes]} & Backlog, higher worse & Non-decreasing \\
\texttt{sum\_requested\_prbs} & Demand, higher worse & Non-decreasing \\
\texttt{tx\_errors downlink (\%)} & Loss, higher worse & Non-decreasing \\
\texttt{dl\_cqi} & Channel, higher better & Non-increasing \\
\texttt{dl\_mcs} & Link rate, higher better & Non-increasing \\
\texttt{rsrp}, \texttt{dl\_snr} (UE join) & Radio quality, higher better & Non-increasing \\
\texttt{dl\_bler} (UE join) & Block error, higher worse & Non-decreasing \\
\midrule
\multicolumn{3}{@{}l}{\emph{Aggressive set (free by default; constrained in ablation)}}\\
\texttt{slice\_prb}, \texttt{sum\_granted\_prbs} & Allocation, plausibly better & Non-increasing \\
\texttt{tx\_brate downlink} & Delivered, leakage-prone & Non-increasing \\
\texttt{num\_ues} & Contention, label-coupled & Non-decreasing \\
\midrule
\texttt{slice\_id}, \texttt{scheduling\_policy} & Context-dependent & Unconstrained \\
\bottomrule
\end{tabular}}
\end{table}
\begin{table}[t]
\centering
\caption{Positive rate of the next-window SLA-violation proxy.}
\label{tab:violation_rates}
\small
\setlength{\tabcolsep}{3pt}
\renewcommand{\arraystretch}{1.03}
\begin{tabular}{@{}llccc@{}}
\toprule
Slice & Positive-label rule & Train & Val & Test \\
\midrule
eMBB  & Rate $<2$\,Mbps              & 0.574 & 0.571 & 0.562 \\
MTC   & Rate $<0.015$\,Mbps          & 0.105 & 0.108 & 0.106 \\
URLLC & Buffer $>28.8$\,B or error $>5\%$ & 0.514 & 0.515 & 0.514 \\
\bottomrule
\end{tabular}
\end{table}

\subsection{Baselines and Metrics}
We compare five models that separate prediction accuracy, federated training, interpretability, and rApp-side constraint governance. Local NAM trains one unconstrained additive model per client and averages performance over local test splits. Centralized NAM trains one unconstrained additive model on pooled data and serves as a non-private upper reference. FedAvg-MLP is a black-box federated baseline trained with standard FedAvg. FedNAM is an unconstrained federated additive model. \modelname{} is the proposed rApp-governed model, using monotone splines on the conservative KPI set and unconstrained subnetworks for the remaining continuous features. All models use the same train/validation/test split: experiments 1--3 for training, experiment 4 for validation, and experiment 5 for testing. Federated models use 30 rounds, $E{=}3$ local epochs, full client participation, batch size 8192, Adam optimization, and positive-class-weighted BCE. Uplink communication is counted as 32-bit client-to-server model updates over all 30 rounds. Class balance varies strongly across slices (10\% for MTC, 57\% for eMBB), so we report average precision (AP) and F1 alongside the area under the ROC curve (AUC). To evaluate rApp-side physical governance, we report monotonicity-violation rate over the constrained KPI set. For each constrained feature, we evaluate the learned shape on a shared grid and count adjacent differences that violate the expected direction; the average violation rate is the mean over constrained features. We also report constrained-shape consistency, defined as the mean pairwise Pearson correlation between per-client local shape functions over constrained features after local training. This measures whether clients agree on the form of each governed KPI effect before aggregation, not the single post-aggregation global shape. Finally, we report parameter count and uplink communication cost to assess whether the governed model remains compact enough for xApp deployment.

\section{Results and Analysis}
\label{sec:results}

All results evaluate the rApp-governed training role of \modelname{}. The Non-RT RIC rApp should produce a global model that remains accurate enough for SLA-risk prediction, physically consistent in its KPI-risk effects, auditable through additive explanations, and compact enough for deployment to a Near-RT RIC xApp. Results are means over three seeds on the full ColO-RAN matrix: three schedulers, 28 allocations, five experiments, and seven base stations.

\paragraph{Accuracy tradeoff of the rApp-produced model.}
Table~\ref{tab:predictive} reports the predictive performance of the federated models. The black-box FedAvg-MLP gives the highest AUC, 0.973. Unconstrained FedNAM reaches 0.943, while \modelname{} reaches 0.901. Thus, enforcing physical consistency costs 0.042 AUC relative to unconstrained FedNAM and 0.072 relative to the black-box MLP. AP and F1 follow the same order. This confirms that the model produced under rApp governance trades some predictive flexibility for the physical-validity guarantees evaluated next.

\begin{table}[t]
\centering
\caption{Predictive performance and model size of the federated models (mean$\pm$std over seeds).}
\label{tab:predictive}
\small
\setlength{\tabcolsep}{3pt}
\renewcommand{\arraystretch}{0.98}
\begin{tabular}{lccc}
\toprule
Metric & FedAvg-MLP & FedNAM (uncon.) & \modelname{} \\
\midrule
AUC & 0.973$\pm$0.0003 & 0.943$\pm$0.002 & 0.901$\pm$0.0001 \\
AP  & 0.960$\pm$0.0004 & 0.916$\pm$0.003 & 0.860$\pm$0.001 \\
F1  & 0.883$\pm$0.0004 & 0.829$\pm$0.002 & 0.777$\pm$0.001 \\
\midrule
Params & 12{,}801 & 7{,}315 & 2{,}563 \\
\bottomrule
\end{tabular}
\end{table}

\paragraph{rApp governance prevents physically invalid explanations.}
Table~\ref{tab:interp} is the central result. It shows why the rApp must govern the learned feature effects rather than only aggregate model updates. Monotonicity does not emerge from the data by itself. The local NAM violates the expected KPI-risk directions at an average rate of 0.461, and the centralized NAM violates them at 0.508. Pooling heterogeneous slice and scheduler regimes into one shape can average toward a physically wrong global direction, so centralized fitting does not reduce violations. Unconstrained federation worsens the problem further: FedNAM reaches an average violation rate of 0.652, because averaging divergent per-client shapes compounds the distortion. The worst cases are severe, including a full inversion of \texttt{dl\_cqi}, where better channel quality is associated with higher predicted risk. In contrast, \modelname{} removes these violations by construction. This result supports the main rApp contribution: the Non-RT RIC rApp does more than train a predictor; it preserves domain-valid KPI-risk behavior under federated aggregation. Fig.~\ref{fig:shape} illustrates this effect for \texttt{dl\_cqi}. The unconstrained federated model learns an inverted channel-quality effect, while \modelname{} enforces the expected non-increasing relationship between channel quality and SLA-risk.

\begin{table}[t]
\centering
\caption{Interpretability: AUC versus monotonicity violation. FedAvg-MLP has no additive shapes.}
\label{tab:interp}
\small
\resizebox{\columnwidth}{!}{%
\begin{tabular}{lccc}
\toprule
Model & AUC & Avg.\ mono.\ viol. & Worst-KPI viol. \\
\midrule
Local NAM & 0.931$\pm$0.001 & 0.461$\pm$0.010 & 0.995 (\texttt{dl\_bler}) \\
Centralized NAM & 0.957$\pm$0.004 & 0.508$\pm$0.045 & 1.000 (\texttt{dl\_cqi}) \\
FedAvg-MLP & 0.973$\pm$0.000 & N/A & N/A \\
FedNAM (uncon.) & 0.943$\pm$0.002 & 0.652$\pm$0.004 & 0.987 (\texttt{dl\_cqi}) \\
\modelname{} & 0.901$\pm$0.0001 & 0.000$\pm$0.000 & 0.000 (\texttt{--}) \\
\bottomrule
\end{tabular}}
\end{table}

\paragraph{Conservative constraint governance is the right operating point.}
The rApp should not constrain every KPI. A sweep over constraint sets confirms this design choice. With no constraints, FedNAM reaches 0.943 AUC and 0.916 AP but has 0.652 average monotonicity violation. Adding only the conservative constraint set removes all violations and raises constrained-shape consistency from 0.712 to 1.000, at a measured cost of 0.042 AUC. Adding the aggressive constraints gives no additional interpretability benefit, because violations are already zero, but it further reduces AUC from 0.901 to 0.836 and AP from 0.860 to 0.771. Therefore, the rApp should govern only the KPI directions with unambiguous wireless meaning and leave contestable features free.

\paragraph{xApp deployment feasibility.}
We assess xApp deployability through model properties rather than a runtime xApp, which we leave to future work. Specifically, we test leakage robustness, scheduler transfer, and communication cost. Because the prediction horizon is one 2.5\,s window, the deployed xApp produces a short-horizon forecast that informs slower policy, so its cadence does not require sub-second near-RT control. Removing the current delivered rate has only a small effect: AUC falls from 0.943 to 0.927 for FedNAM and from 0.901 to 0.892 for \modelname{}. The result is therefore not driven only by throughput autocorrelation. In a leave-one-scheduler-out check, \modelname{} reaches a mean AUC of 0.887, compared with 0.927 for unconstrained FedNAM and 0.965 for FedAvg-MLP, showing that governed KPI-risk relationships transfer to an unseen scheduler with the same accuracy cost observed in the main split.

Finally, the rApp-side update cost is lower for \modelname{} by construction. Uplink traffic is the model update size times the number of rounds, and the constrained representation shrinks the update itself: each constrained feature is transmitted as a 16-value spline-increment vector instead of the several hundred weights of an MLP subnetwork. This reduces the parameter count from 7{,}315 (FedNAM) and 12{,}801 (FedAvg-MLP) to 2{,}563 (Table~\ref{tab:predictive}), and the total client-to-server uplink over 30 rounds falls in proportion, from 6.14\,MB and 10.75\,MB to 2.15\,MB. The rApp therefore produces a physically governed and auditable model with lower training-update overhead before xApp deployment.
\begin{figure}[t]
\centering
\includegraphics[width=0.8\columnwidth]{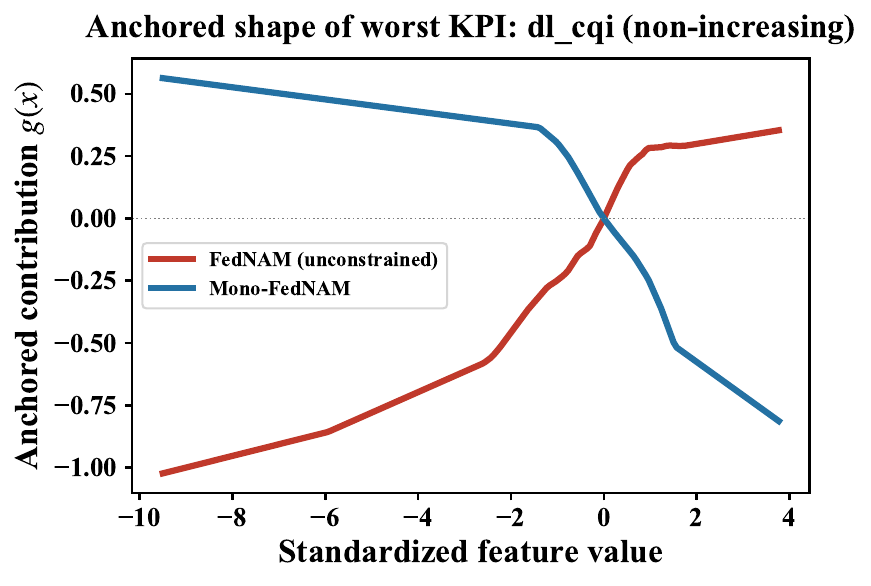}
\caption{Unconstrained FedNAM learns an inverted channel-quality effect; Monotone FedNAM enforces the physically expected non-increasing shape.}
\label{fig:shape}
\end{figure}

\section{Discussion and Conclusion}
\label{sec:discussion}
\label{sec:conclusion}

We presented \modelname{}, an rApp-governed federated additive model for per-slice SLA-risk prediction in O-RAN, whose global model deploys to a Near-RT RIC xApp for online inference with additive explanations. Three findings stand out. First, interpretability alone does not ensure physical validity: unconstrained NAMs violate expected KPI-risk directions in every tested regime, and non-IID federation worsens the violations. Second, governing the conservative constraint set removes all violations, aligns client shapes, and lowers uplink cost, at about 0.04 AUC against unconstrained FedNAM and 0.07 against the black-box MLP. Third, constraining the contestable KPIs as well reduces accuracy without interpretability gain, so conservative governance is the right operating point. The model also transfers to an unseen scheduling policy. The study covers one emulated testbed with static users, a proxy URLLC label, and seven clients in a single operator domain, so the multi-tenant confidentiality setting is motivated rather than instantiated. Future work will test unseen resource allocations, dynamic mobility, larger federations, and governance across operator trust domains, where no single SMO exists.

\bibliographystyle{IEEEtran}
\bibliography{references}

@inproceedings{mcmahan2017fedavg,
 title={Communication-efficient learning of deep networks from decentralized data},
  author={McMahan, Brendan and Moore, Eider and Ramage, Daniel and Hampson, Seth and y Arcas, Blaise Aguera},
  booktitle={Artificial intelligence and statistics},
  pages={1273--1282},
  year={2017},
  organization={Pmlr}
}

@article{hastie1986gam,
title={Generalized additive models},
  author={Hastie, Trevor J},
  journal={Statistical models in S},
  pages={249--307},
  year={2017},
  publisher={Routledge}
}

@inproceedings{agarwal2021nam,
author = {Agarwal, Rishabh and Melnick, Levi and Frosst, Nicholas and Zhang, Xuezhou and Lengerich, Ben and Caruana, Rich and Hinton, Geoffrey E.},
title = {Neural additive models: interpretable machine learning with neural nets},
year = {2021},
isbn = {9781713845393},
publisher = {Curran Associates Inc.},
address = {Red Hook, NY, USA},
booktitle = {Proceedings of the 35th International Conference on Neural Information Processing Systems},
articleno = {359},
numpages = {13},
series = {NIPS '21}
}

@article{nanda2025fednam,
 title={FedNAMs: Performing interpretability analysis in federated learning context},
  author={Nanda, Amitash and Balija, Sree Bhargavi and Sahoo, Debashis},
  journal={arXiv preprint arXiv:2506.17466},
  year={2025}
}

@book{goodfellow2016deep,
 title={Deep learning},
  author={Goodfellow, Ian and Bengio, Yoshua and Courville, Aaron and Bengio, Yoshua},
  volume={1},
  number={2},
  year={2016},
  publisher={MIT press Cambridge}
}

@article{polese2023understanding,
  author={Polese, Michele and Bonati, Leonardo and D'Oro, Salvatore and Basagni, Stefano and Melodia, Tommaso},
  title={Understanding {O-RAN}: Architecture, Interfaces, Algorithms, Security, and Research Challenges},
  journal={IEEE Communications Surveys \& Tutorials}, volume={25}, number={2}, pages={1376--1411}, year={2023}}

@article{coloran2022,
  author={Polese, Michele and Bonati, Leonardo and D'Oro, Salvatore and Basagni, Stefano and Melodia, Tommaso},
  title={{ColO-RAN}: Developing Machine Learning-Based {xApps} for Open {RAN} Closed-Loop Control on Programmable Experimental Platforms},
  journal={IEEE Transactions on Mobile Computing}, volume={22}, number={10}, pages={5787--5800}, year={2023}}

@inproceedings{foukas2023ran,
  author={Foukas, Xenofon and Radunovic, Bozidar and Balkwill, Matthew and Lai, Zhihua},
  title={Taking 5{G} {RAN} Analytics and Control to a New Level},
  booktitle={Proc. ACM MobiCom}, year={2023}}

@inproceedings{raftopoulos2024latency,
  author={Raftopoulos, Raoul and D'Oro, Salvatore and Melodia, Tommaso and Schembra, Giovanni},
  title={{DRL}-Based Latency-Aware Network Slicing in {O-RAN} with Time-Varying {SLAs}},
  booktitle={Proc. IEEE Int. Conf. Computing, Networking and Communications (ICNC)}, year={2024}}

@article{kairouz2021advances,
  author={Kairouz, Peter and McMahan, H. Brendan and others},
  title={Advances and Open Problems in Federated Learning},
  journal={Foundations and Trends in Machine Learning}, volume={14}, number={1--2}, pages={1--210}, year={2021}}

@inproceedings{zhang2022federated,
  author={Zhang, Han and Zhou, Hao and Erol-Kantarci, Melike},
  title={Federated Deep Reinforcement Learning for Resource Allocation in {O-RAN} Slicing},
  booktitle={Proc. IEEE Global Communications Conference (GLOBECOM)}, year={2022}}

@inproceedings{erdol2022federated,
  author={Erdol, Hakan and Wang, Xiaoyang and Li, Pengwei and Thomas, Jonathan D. and Piechocki, Robert and Oikonomou, George and Inacio, Rui and Ahmad, Abdelrahim and Briggs, Keith and Kapoor, Shipra},
  title={Federated Meta-Learning for Traffic Steering in {O-RAN}},
  booktitle={Proc. IEEE Vehicular Technology Conference (VTC2022-Fall)}, year={2022}}

@inproceedings{attanayaka2023p2p,
  author={Attanayaka, Dinaj and Porambage, Pawani and Liyanage, Madhusanka and Ylianttila, Mika},
  title={Peer-to-Peer Federated Learning Based Anomaly Detection for Open Radio Access Networks},
  booktitle={Proc. IEEE Int. Conf. Communications (ICC)}, year={2023}}

@article{brik2024xai,
  author={Brik, Bouziane and Chergui, Hatim and Zanzi, Lanfranco and Devoti, Francesco and Ksentini, Adlen and Siddiqui, Muhammad Shuaib and Costa-P\'erez, Xavier and Verikoukis, Christos},
  title={Explainable {AI} in 6{G} {O-RAN}: A Tutorial and Survey on Architecture, Use Cases, Challenges, and Future Research},
  journal={IEEE Communications Surveys \& Tutorials}, year={2024}}

@article{khan2024explainable,
  author={Khan, Nasir and Coleri, Sinem and Abdallah, Asmaa and Celik, Abdulkadir and Eltawil, Ahmed M.},
  title={Explainable and Robust Artificial Intelligence for Trustworthy Resource Management in 6{G} Networks},
  journal={IEEE Communications Magazine}, volume={62}, number={4}, pages={50--56}, year={2024}}

@article{fiandrino2023explora,
  author={Fiandrino, Claudio and Bonati, Leonardo and D'Oro, Salvatore and Polese, Michele and Melodia, Tommaso and Widmer, Joerg},
  title={{EXPLORA}: {AI/ML} Explainability for the Open {RAN}},
  journal={Proceedings of the ACM on Networking (CoNEXT)}, volume={1}, pages={1--26}, year={2023}}

@inproceedings{rezazadeh2023explanation,
  author={Rezazadeh, Farhad and Chergui, Hatim and Mangues-Bafalluy, Josep},
  title={Explanation-Guided Deep Reinforcement Learning for Trustworthy 6{G} {RAN} Slicing},
  booktitle={Proc. IEEE Int. Conf. Communications Workshops (ICC Workshops)}, year={2023}}

@inproceedings{sill1997monotonic,
  author={Sill, Joseph},
  title={Monotonic Networks},
  booktitle={Advances in Neural Information Processing Systems (NeurIPS)}, year={1997}}

@inproceedings{liu2020certified,
  author={Liu, Xingchao and Han, Xing and Zhang, Na and Liu, Qiang},
  title={Certified Monotonic Neural Networks},
  booktitle={Advances in Neural Information Processing Systems (NeurIPS)}, year={2020}}

@inproceedings{runje2023monotonic,
  title={Constrained monotonic neural networks},
  author={Runje, Davor and Shankaranarayana, Sharath M},
  booktitle={International Conference on Machine Learning},
  pages={29338--29353},
  year={2023},
  organization={PMLR},
}

@article{samdanis2016slicebroker,
  author={Samdanis, Konstantinos and Costa-Perez, Xavier and Sciancalepore, Vincenzo},
  journal={IEEE Communications Magazine}, 
  title={From network sharing to multi-tenancy: The 5G network slice broker}, 
  year={2016},
  volume={54},
  number={7},
  pages={32-39},
  doi={10.1109/MCOM.2016.7514161}}

@techreport{gsma2025ng116,
  author      = {{GSMA}},
  title       = {Generic Network Slice Template},
  number      = {NG.116, Version 10.0},
  institution = {GSM Association},
  type        = {Official Document},
  year        = {2024},
  url         = {https://www.gsma.com/newsroom/gsma_resources/ng-116-generic-network-slice-template-v8-0/}
}

\end{document}